\documentclass[letterpaper, 10pt, conference]{ieeeconf}
\usepackage{amsmath,amssymb,euscript,yfonts,psfrag,latexsym,dsfont,graphicx}
\usepackage{bbm,color,amstext,wasysym,parskip,balance,cite}
\graphicspath{{./},{./figures/}}

\usepackage[utf8]{inputenc} 
\usepackage[T1]{fontenc}    
\usepackage{hyperref}       
\usepackage{url}            
\usepackage{booktabs}       
\usepackage{amsfonts}       
\usepackage{nicefrac}       
\usepackage{microtype}      
\usepackage{algorithm}
\usepackage{algorithmic}
\usepackage{amsmath}
\usepackage{graphicx}
\usepackage{subcaption}
\usepackage{xcolor}

\newcommand{\ignore}[1]{}

\definecolor{grey}{rgb}{0.6,0.3,0.3}
\definecolor{lgrey}{rgb}{0.9,.7,0.7}

\IEEEoverridecommandlockouts
\newtheorem{theorem}{Theorem}

\newcommand{\mR}{{\mathbb R}}

\newcommand{\mE}{{\mathbb E}}

\def\spacingset#1{\def\baselinestretch{#1}\small\normalsize}
\spacingset{1}

\title{Probabilistic Kernel Support Vector Machines}
\author{Yongxin Chen, Tryphon T. Georgiou, and Allen R. Tannenbaum
\thanks{Y.\ Chen is with the School of Aerospace Engineering, Georgia Tech, Atlanta, GA; yongchen@gatech.edu}
\thanks{T.T.\ Georgiou is with the Department of Mechanical and Aerospace Engineering, UCI, Irvine, CA; tryphon@uci.edu}
\thanks{A.R. Tannenbaum is with Departments of Computer Science and Applied Mathematics \& Statistics, Stony Brook University, NY; arobertan@cs.stonybrook.edu}}

\begin{document}
\maketitle
\begin{abstract}
We propose a probabilistic enhancement of standard kernel Support Vector Machines for binary classification, in order to address the case when, along with given data sets, a description of uncertainty (e.g., error bounds) may be available on each datum. In the present paper, we specifically consider Gaussian distributions to model uncertainty.  Thereby, our data consist of pairs $(x_i,\Sigma_i)$, $i\in\{1,\ldots,N\}$, along with an indicator $y_i\in\{-1,1\}$ to declare membership in one of two categories for each pair. These pairs may be viewed to represent the mean and covariance, respectively, of random vectors $\xi_i$ taking values in a suitable linear space (typically $\mR^n$). Thus, our setting may also be viewed as a modification of Support Vector Machines to classify distributions, albeit, at present, only Gaussian ones.
We outline the formalism that allows computing suitable classifiers via a natural modification of the standard ``kernel trick.'' The main contribution of this work is to point out a suitable kernel function for applying Support Vector techniques to the setting of uncertain data for which a detailed uncertainty description is also available  (herein, ``Gaussian points'').
\end{abstract}

\section{Introduction and motivation}

The Support Vector methodology is a nonlinear generalization of Linear Classification, typically utilized for binary classification. In this methodology, a training set of data becomes available where each datum belongs to one of two categories while an indicator function identifies the category for each. Imagine points on a plane labeled with two different colors. Solely based on their location and proximity to others of particular colors, one needs to classify new points and assign them to one of the two categories. When the points of the training set in the two categories can be separated by a line, then this line can be chosen to separate two respective regions, one corresponding to each of the two categories. On the other hand, when the ``training'' points of the two categories cannot be separated by a line, then a higher order curve needs to be chosen. The Support Vector methodology provides a rigorous and systematic way to do exactly that.

In the case where no simple boundary can be chosen to delineate regions for the points in the two categories, and a simple classification rule is still sought based on location, then a smooth boundary can be selected to define regions that contain the vast majority of points in the respective categories. Thus, simplicity of the classification rule (Occam's razor) may be desirable, as the domains corresponding to the two categories may not be completely separate and a misclassification error in that case must be deemed acceptable. This again can be handled by a suitable relaxation of Support Vector Machines (SVM). At present the rather extensive literature on the subject continues to expand unabated, and the same applies to the growing library of options with regard to formulations and related methodologies \cite{SteChr08,chang2011libsvm,Suy01}, as well as to applications in system identification and decision making \cite[Chapter 15]{kung2014kernel}.

Initially, algorithms were developed assuming that samples are error-free. However, evidently, in many practical applications such an assumption is unrealistic. To this end, a number of approaches have been proposed. For a survey and other directions in developing SVM's designed to address uncertain data sets, due to sampling, modeling, or instrumentation errors, or class label uncertainty, see \cite{wang2014survey,bi2005support,trafalis2007robust,xie2018uncertain} and the references therein.

The purpose of this paper is to add to this growing literature by providing yet another angle to how one can incorporate uncertainty in data sets to SVMs. Specifically, we consider the paradigm where measurements are provided along with error bars that quantify expected margins. This is the case when uncertain is recorded at the time data is collected (e.g., by keeping track of sensors being used or physical conditions that hinder precision). Alternatively, a datum may be an empirical probability distribution, which in turn may be abstracted as the mean along with a covariance matrix to keep track of the spread. Either way, we assume that our data points are approximated by a Gaussian distribution. Thus, in such cases, a ``datum'' is now a point in a space that is more complex than $\mathbb R^n$. We seek to devise suitable kernels, and thereby extend the Support Vector methodology accordingly. The non-trivial nature of the problem to devise kernels for data-points on curved metric spaces has been discussed in \cite{feragen2015geodesic}. Yet, as we will see, by introducing a suitable probability law that couples the dataset and thereby takes advantage of the available covariance information, the case of ``Gaussian data points'' allows for a rather simple probabilistic kernel (Theorem 1).

Thus, we postulate that measurements provide vectors $x\in\mathbb R^n$ together with covariance matrices $\Sigma$ that quantify our uncertainty in the value $x$ being recorded. For the purposes of binary classification, we record such pairs $(x,\Sigma)$ along with the information on the category of origin of each, which is given by an indicator value $y=\pm 1$. A datum with a large variance naturally delineates a larger volume that should be associated to the corresponding category, or at least, impact proportionately the drawing of the separation between the two. Prior state-of-the-art does not do that. Thus, in the present note, we propose a new kernel that allows treating measurement uncertainty as part of the data.

{\bf Related work:} One approach to incorporate data uncertainties in SVM is through robust optimization \cite{bi2005support,trafalis2007robust,wang2014survey}, and many related works \cite{wang2014survey} focus on linear SVM. In the robust optimization approach, uncertainties are modeled by a hard bound instead of a soft probability bound. The work that is closest to ours is \cite{xie2018uncertain} where the authors also defined a kernel for uncertain measurements. However, their kernel is defined over histograms which require discretization for Gaussian measurements, and thus do not scale for high dimensional applications.

Below, in Section II, we highlight some background on Support Vector Machines. Section III gives our main result that identifies a kernel function for the case of data provided in the form of ``Gaussian points.'' Several numerical examples are presented in Section \ref{sec:example} to illustrate our framework. It follows by a short concluding remark in Section \ref{sec:conclusion}.

\section{Background}
Support Vector Machines (SVMs) constitute a well-established technique for (binary) classification and regression analysis \cite{scholkopf1999advances}. The main idea is to embed a given (training) data set into a high dimensional space $\mathcal H$, of dimension much higher than the native dimension of the base space $\mathcal X$ and possibly infinite, so that for binary classification, the two classes can be separated with a hyperplane \cite{cover1965geometrical}. Effectively, this {\em separation hyperplane} projects down to the native base space $\mathcal X$, where the data points originally reside, as curves/surfaces/etc.\ that separate the two classes. The imbedding into $\mathcal H$ is effected by a mapping
\[
\varphi\,:\, x\in{\mathcal X} \mapsto \varphi(x)\in{\mathcal H},
\]
where $\varphi(x)$ is referred to as the {\em feature vector}. The space $\mathcal H$ has an inner product structure (Hilbert space) and, naturally, the construction of the separating hyperplane relies on the geometry of $\mathcal H$. However, and most importantly, the map $\varphi$ does not need to be known explicitly and does not need to be applied to the data. All necessary operations, inner product and projections that show up in the classifier and computations, can be effectively carried out in the base space using the so called ``kernel trick.''

Indeed, to accomplish the task and construct the classification rule (via suitable curves/surfaces/etc.), it is sufficient to know the {\em kernel}
\begin{equation}\label{eq:kernel}
k(x,y):= \langle \varphi(x),\varphi(y)\rangle_{\mathcal H},
\end{equation}
with $x,y\in\mathcal X$. It evaluates the inner product in the feature space as a function of the corresponding representatives $x,y$ in the native base space.
Thus, the kernel is a bivariate function which is completely characterized by the property of being positive, in the sense that
for all $x_i\in \mathcal X$, $i\in\{1,\ldots,N\}$, and any corresponding set of values $\alpha_i\in\mathbb R$ (as we are interested in real-valued kernels)
\[
\sum_{i,j=1}^N \alpha_i\alpha_j k(x_i,x_j)\geq 0.
\]
Necessity stems from the fact that the left hand side above is
\[
\langle \sum_{i=1}^N\alpha_i\varphi(x_i),\sum_{j=1}^N\alpha_j\varphi(x_j)\rangle_{\mathcal H}= \|\sum_{i=1}^N\alpha_i\varphi(x_i)\|^2_{\mathcal H}\geq 0.
\]
Sufficiency, in the existence of a feature map $\varphi$ that realizes \eqref{eq:kernel} is a celebrated theorem \cite{aronszajn1950theory}, see also \cite[Theorem 7.5.2]{alpay2015advanced} and \cite[page 30, Theorem 3.1 (Mercer)]{scholkopf1999advances}.

Classification relies on constructing a classifier that is built on a linear functional, when viewed in $\mathcal H$.
It is of the form
\begin{equation}
x\to {\rm sign}( \langle w,\varphi(x)\rangle_{\mathcal H}-b),
\end{equation}
and the value $\pm 1$ aims to differentiate between elements in two categories.
The coefficients $w\in\mathcal H$ and $b\in\mathbb R$ are chosen so that
\[
\{ h\in\mathcal{H} \mid \langle w,h\rangle_{\mathcal H}-b=0\}
\]
is a separating hyperplane of the two subsets
\[
\mathcal{S}_\pm =\{ \varphi(x_i) \mid y_i=\pm 1\}
\]
of the compete (training) data set. Once again, $\varphi(\cdot)$ does not need to evaluated at any point in the construction, existence of such a map is enough, and it is guaranteed by the positivity of the kernel.

The construction of the classifier requires selection of the parameters $w=\sum_{i=1}^Nc_iy_i\varphi(x_i) \in\mathcal H$ and $b\in \mathbb R$. These are chosen
either (``hard margin'') to minimize
\[
\langle w,w\rangle_{\mathcal H} \mbox{ subject to } y_i(\langle w,\varphi(x_i)\rangle_{\mathcal H}-b)\geq 1,
\]
or, to (``soft margin'') minimize
\begin{equation}\label{eq:soft}
\frac{1}{N}\sum_{i=1}^N \max\{0, 1-y_i(\langle w,\varphi(x_i)\rangle_{\mathcal H}-b)\} +\lambda  \langle w,w\rangle_{\mathcal H},
\end{equation}
for all available points in the ``training set.'' The ``hard margin'' formulation coincides with the limit of the ``soft'' formulation as $\lambda\to 0$ when the two clusters are separable. The dual formulation of \eqref{eq:soft} becomes the problem of maximizing
\[
\sum _{i=1}^N c_i-{\frac {1}{2}}\sum _{i,j=1}^N
y_{i}y_{j}c_{i}c_{j}k(x_{i}, x_{j}),
\]
subject to
\[
\sum_{i=1}^Nc_iy_i=0, \mbox{ as well as } 0\leq c_i\leq (2N\lambda)^{-1}
\]
for all $i$.
The coefficients $c_i$ can now be obtained via quadratic programming, and $b$ can be found as \[
b=\sum_{i=1}^N c_iy_ik(x_i,x_j) - y_j
\]
with $j$ corresponding to an index such that $0<c_j<(2N\lambda)^{-1}$. The classification rule then becomes
\begin{equation}
x\to {\rm sign}( \sum_{i=1}^Nc_iy_ik(x_i,x)-b),
\end{equation}
The above follows standard development, see e.g., \cite{scholkopf2001learning} as well as the Wikipedia webpage on Support Vector Machines \cite{wiki}.

\section{Classification of Gaussian points}
The problem we are addressing in the present note is the classification of {\em uncertain data points} into one of two categories, i.e., a binary classification as before. However, a salient feature of our setting is that data are only known with finite accuracy.
Uncertainty is modeled in a probabilistic manner. For simplicity, in this paper, we consider only {\em Gaussian points}. These consist of pairs $(x_i,\Sigma_i)$ representing the mean and variance of a normally distributed random vector $\xi_i$. Regardless of the simplicity, Gaussian points are sufficiently general to cover many real application since most physical measurements do involve random noise that satisfies a Gaussian distribution.

Alternatively, we may think of the data as points on a manifold of distributions (though, only Gaussian at present). These may represent approximations of empirical distributions that have been obtained at various times. An indicator $y_i$ is provided
as usual along with the information of the category that the current datum belongs to.
If we regard it as representing a distribution, we postulate that it arose from experiment involving population $y_i$.

We follow the standard setting of {\em kernel Support Vector Machines} (kSVMs) that was outlined in the background section, which we overlay a probabilistic component.
This {\em Probabilistic} kernel Support Vector Machine (PkSVM)
relies on a suitable modification of the kernel.
To this end we consider the set of ``data points''
\[
\Omega :=\{(x,\Sigma)\mid x\in \mR^n,\,\Sigma\in S_{+,n}\}
\]
with $S_{+,n}$ the cone of non-negative definite symmetric matrices in $\mR^{n\times n}$. We also consider
the family of normally distributed random vectors
\[
\xi\sim N(x,\Sigma)\mbox{ for }(x,\Sigma)\in\Omega.
\]
We utilize the popular exponential Radial Basis Function (RBF) kernel
\begin{equation}\label{eq:gaussian}
k(x,y)=e^{-\frac{1}{2\sigma^2}\|x-y\|^2},
\end{equation}
where the parameter $\sigma>0$ affects the scale of the desired resolution. For any $(x_i,\Sigma_i)\in\Omega, (x_j,\Sigma_j) \in \Omega$, we define random vectors
	\begin{subequations}
	\begin{eqnarray}
	\xi_i &=& x_i + \Sigma_i^{1/2} \epsilon
	\\
	\xi_j &=& x_j + \Sigma_j^{1/2} \epsilon
	\end{eqnarray}
	\end{subequations}
where $\epsilon$ is a zero-mean Gaussian random vector with unit covariance, i.e., $\epsilon\sim N(0,I)$.
Apparently, $\xi_i$ has distribution $N(x_i,\Sigma_i),\,i=1,2$. We then define the kernel
\begin{align}\nonumber
&\kappa((x_i,\Sigma_i),(x_j,\Sigma_j))
\\&=\mE\{k(\xi_i,\xi_j)\}\nonumber\\
\nonumber&=\mE\{e^{-\frac{1}{2\sigma^2}\|\xi_i-\xi_j\|^2}\}
\\\nonumber
&= (2\pi)^{-n/2}\int e^{-\frac{1}{2\sigma^2} \|x_i-x_j+\Sigma_i^{1/2} \epsilon-\Sigma_j^{1/2} \epsilon\|^2-\frac{1}{2} \|\epsilon\|^2}
d\epsilon\\
&= |I+U_{ij}^2|^{-1/2}\times e^{-\frac{1}{2\sigma^2}
\|x_i-x_j\|^2_{(I+U_{ij}^2)^{-1}}},
\label{eq:ij}
\end{align}
where
	\begin{equation}\label{eq:U}
		U_{ij} = \frac{\Sigma_i^{1/2}-\Sigma_j^{1/2}}{\sigma}.
	\end{equation}
We now state our main result.

\begin{theorem}\label{thm1} The function
\begin{equation}\label{eq:kernel1}
\kappa((x_i,\Sigma_i),(x_j,\Sigma_j))=
|I+U_{ij}^2|^{-1/2}\times e^{-\frac{1}{2\sigma^2}
\|x_i-x_j\|^2_{(I+U_{ij}^2)^{-1}}}
\end{equation}
with $U_{ij}$ in \eqref{eq:U} defines a positive kernel on $\Omega$.
\end{theorem}


\begin{proof} Consider $k(x,y)$ in \eqref{eq:gaussian} for  $x,y\in\mR^n$. Then, for any collection $\alpha_i\in \mathbb R$ and any collection of $(x_i,\Sigma_i)\in \Omega$, for $i\in\{1,\ldots,N\}$,
\begin{align*}
&\sum_{i,j=1}^N \alpha_i\alpha_j \kappa((x_i,\Sigma_i),(x_j,\Sigma_j))\\
&\hspace*{1cm} =\mE\{\sum_{i,j=1}^N \alpha_i\alpha_j k(\xi_i,\xi_j)\}\\
&\hspace*{1cm} =\mE\{
\langle \sum_{i=1}^N \alpha_i\varphi(\xi_i),
            \sum_{j=1}^N \alpha_j\varphi(\xi_i)
            \rangle_{\mathcal H}\} \geq 0,
\end{align*}
with $\varphi(\cdot)$ the map to the radial basis functions corresponding to the positive
kernel \eqref{eq:gaussian}. The claim in the theorem follows.
\end{proof}

The kernel \eqref{eq:kernel1} partially encapsulates the topology of the manifold $\Omega$. Indeed, when $\sigma$ is relative large compared with $\Sigma_i,\Sigma_j$, then $\kappa((x_i,\Sigma_i),(x_j,\Sigma_j))$ decreases as the difference between $\Sigma_i$ and $\Sigma_j$ increases. Thus, $(x_i,\Sigma_i)$ and $(x_j,\Sigma_j)$ become further to each other in the feature space.

With the probabilistic kernel $\kappa (\cdot,\cdot)$, we can then formulate the PkSVM over the space of Gaussian data points space $\Omega$. In particular, given $N$ data points $\{(x_i,\Sigma_i), y_i\}$, the dual formulation of PkSVM reads
	\begin{subequations}\label{eq:PkSVM}
	\begin{eqnarray}
	\hspace{-0.5cm}\max_{c_1,\ldots,c_N} &&\!\!\!\sum _{i=1}^N c_i\!-\!{\frac {1}{2}}\!\sum _{i,j=1}^Ny_{i}y_{j}c_{i}c_{j}\kappa ((x_{i},\Sigma_i),(x_{j},\!\Sigma_j))
	\\
	\mbox{s.t.} &&
	\sum_{i=1}^Nc_iy_i=0,
	\\&& 0\leq c_i\leq (2N\lambda)^{-1},~~1\le i \le N.
	\end{eqnarray}
	\end{subequations}
Once the coefficients $\{c_1,c_2,\ldots,c_N\}$ are learned, for any point $(x,\Sigma)\in \Omega$, the classification rule is
\begin{equation}
x\to {\rm sign}( \sum_{i=1}^Nc_iy_i \kappa((x_i,\Sigma_i), (x,\Sigma))-b),
\end{equation}
where
	\begin{equation}
	b=\sum_{i=1}^N c_iy_i\kappa((x_i,\Sigma_i),(x_j,\Sigma_j)) - y_j.
	\end{equation}
It can be seen that when applied to points in $\Omega$ having zero uncertainly, i.e., when the corresponding covariance matrices are identically $0$, then the Probabilistic kernel Support Vector Machine model reduces to the standard one where points lie in $\mR^n$. That is,
\[
\kappa((x_i,0),(x_j,0))=k(x_i,x_j).
\]
Thereby, PkSVM is natural extension of standard SVM, able to account for error and uncertainty that is available and encoded in the data.

\section{Numerical Example}\label{sec:example}
In this section, we present a simple example to highlight the PkSVM framework we developed. Consider a synthetic dataset generated as follows. One cluster, labeled by $y_i = 1$, consists of 200 data points $\{x_i\,\mid\,1\le i\le 200\}$ uniformly sampled over a 2D unit disk. The covariance for each of these samples is set to be
	\[
		\Sigma_i=\Sigma_L = \left[\begin{matrix}
		0.01 & 0 \\ 0 & 0.01
		\end{matrix}\right],
	\]
that is, they are of low uncertainty. Another cluster, labeled by $y_i=-1$, consists of 200 data points $\{x_i\,\mid\,201\le i \le 400\}$ uniformly sampled over the annulus $\{x\,\mid\, 1\le \|x\|\le 2\}$. Their covariances are set to be
	\[
		\Sigma_i=\Sigma_H = \left[\begin{matrix}
		0.09 & 0 \\ 0 & 0.09
		\end{matrix}\right],
	\]
that is, they have higher uncertainties. This dataset $Z=\{(x_i,\Sigma_i),y_i\,\mid\, 1\le i\le 400\}$ is depicted in Figure \ref{fig:dataset}.
\begin{figure}[h]
\centering
\includegraphics[width=0.4\textwidth]{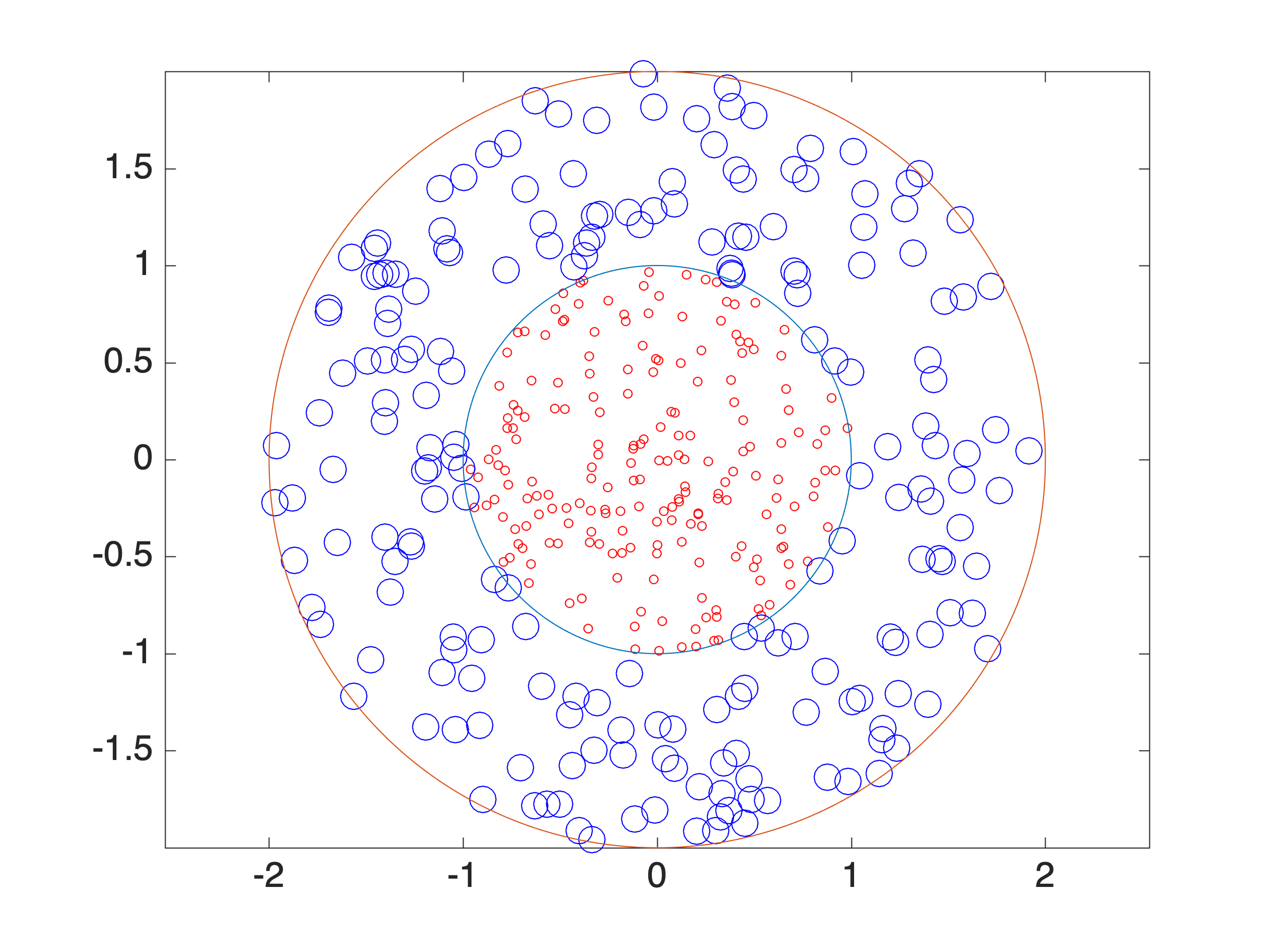}
\caption{The intensity of uncertainty of the data points are captured by the size of the circles. The red circles denote the covariances of the cluster with label 1. The blue circles denote the covariances of the cluster with label -1. The blue data points have higher uncertainties.}
\label{fig:dataset}
\end{figure}

We train the PkSVM model \eqref{eq:PkSVM} with the above dataset $Z$. The parameters are set to be $\lambda = 0.001, \sigma = 1$. To illustrate the trained model, we apply it to predict the score of any test data $(x,\Sigma)\in\Omega$ with $x\in [-2,\,2]\times[-2,\,2]$. In particular, we fix $\Sigma$ and predict the score of $(x,\Sigma)$ over a grid. When $\Sigma=\Sigma_L$, the result is illustrated in Figure \ref{fig:lowcov}. When $\Sigma=\Sigma_H$, the result is shown in Figure \ref{fig:highcov}. The three black curves are contours of the scores corresponding to values -0.5, 0, 0.5 respectively. Recall that the score is always between -1 and 1, and a score closer to 1 (-1) indicates that the data point is more likely to have label 1 (-1).
\begin{figure}[h]
\centering
\includegraphics[width=0.4\textwidth]{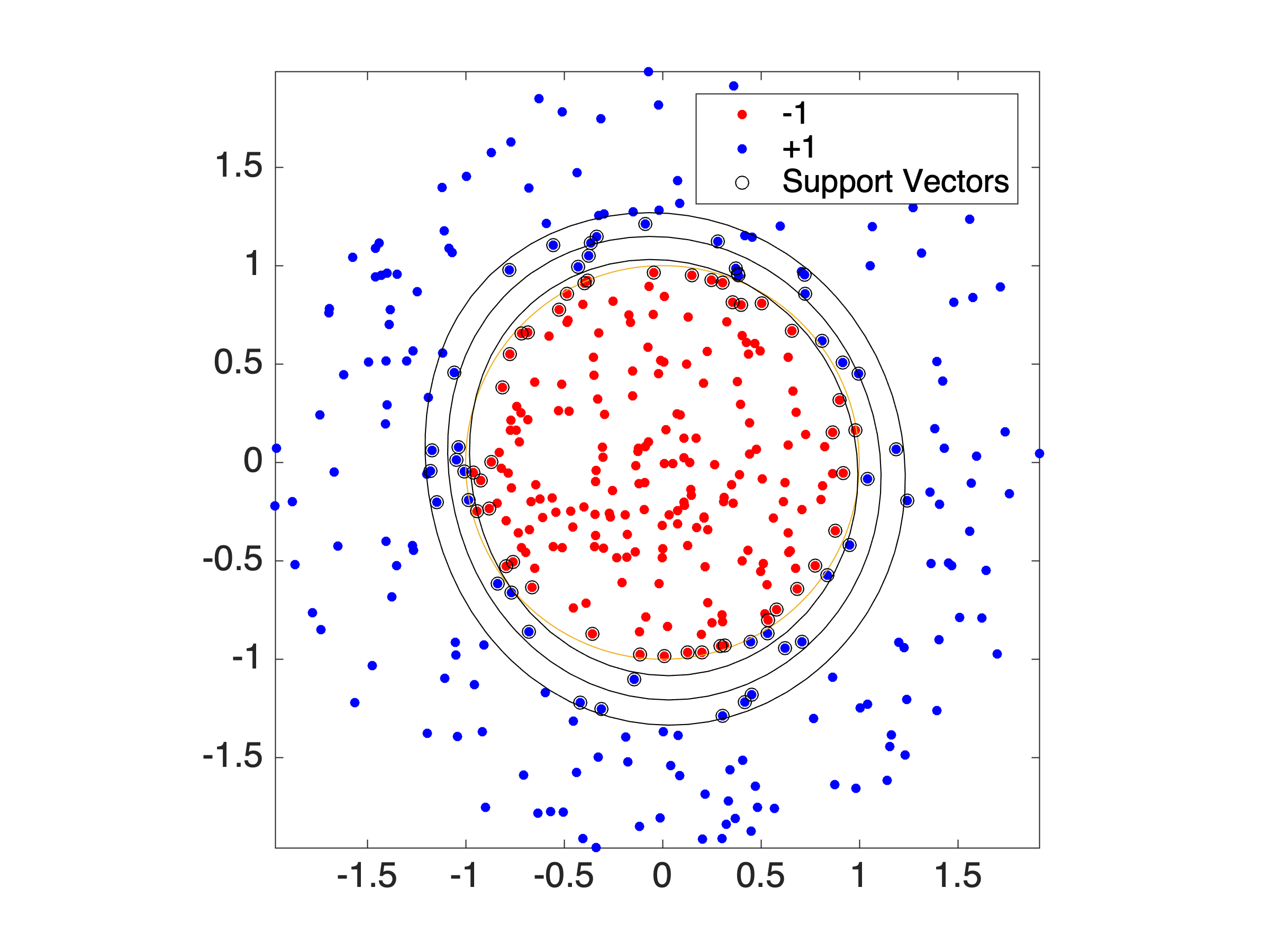}
\caption{Prediction contours of testing data with low uncertainty $\Sigma_L$}
\label{fig:lowcov}
\end{figure}
\begin{figure}[h]
\centering
\includegraphics[width=0.4\textwidth]{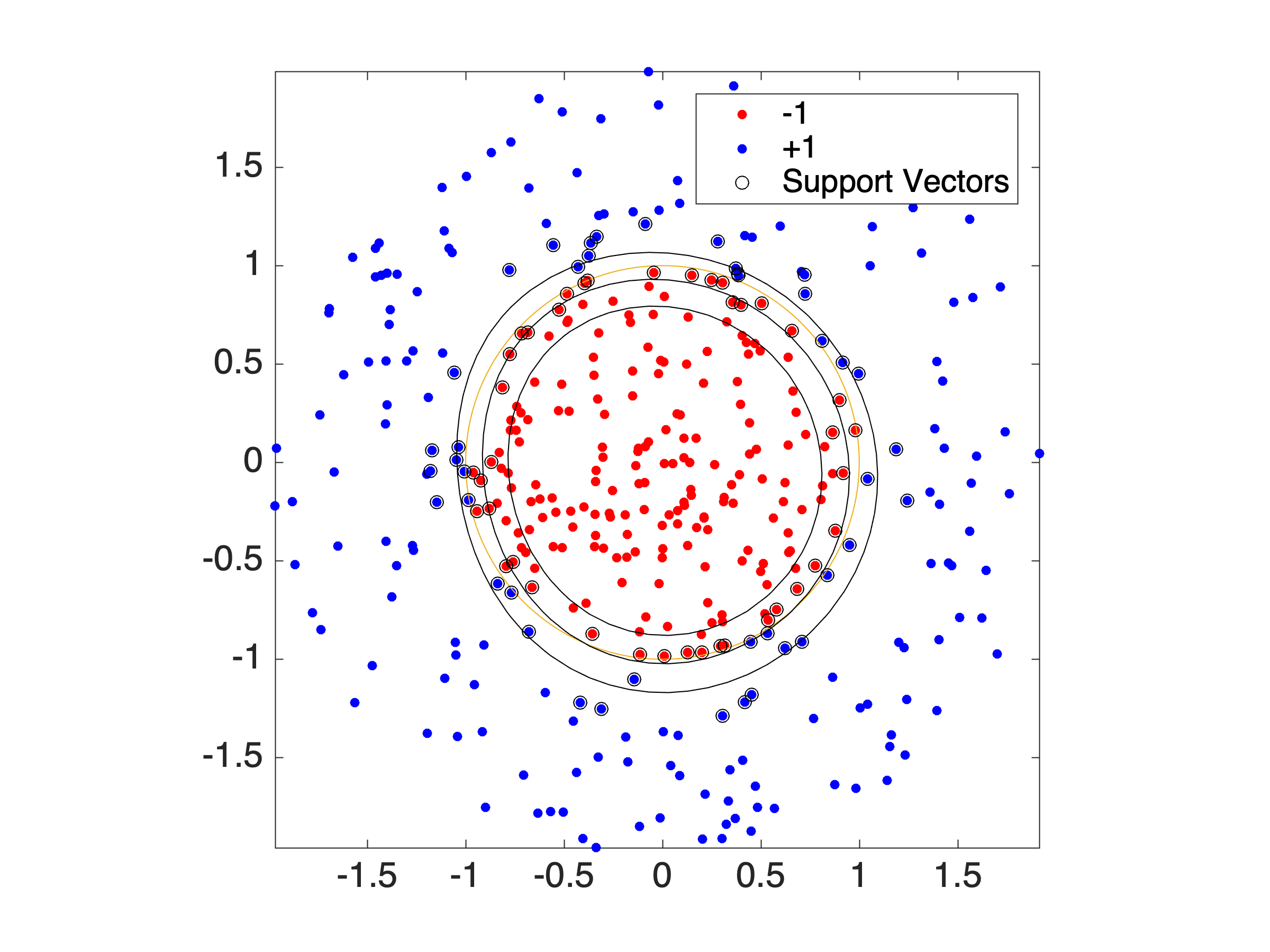}
\caption{Prediction contours of testing data with low uncertainty $\Sigma_H$}
\label{fig:highcov}
\end{figure}

As can be observed from Figure~\ref{fig:lowcov}-\ref{fig:highcov}, when $\Sigma=\Sigma_H$, the boundary contours tend to go inward the disk. This is consistent with the intuition of the PkSVM framework. Indeed, since $\Sigma=\Sigma_H$, with our kernel \eqref{eq:kernel1}, the data on the boundary shows higher similarity with those in the $y_i=-1$ cluster than those in the $y_i=1$ cluster. Intuitively, the data point is closer to the cluster $y_i=-1$ in the feature space. Thus, data with higher uncertainty will be more likely to be labeled by -1. Similarly, the boundary contours tend to go outward the disk when $\Sigma=\Sigma_L$; the rationale is exactly the same.

For the sake of comparison, we also run a standard SVM (Figure \ref{fig:SVM}) over the same dataset $\{x_i, y_i\,\mid\,1\le i\le 400\}$ without the covariances $\{\Sigma_i\}$. Note that standard SVM cannot process the extra information introduced by the covariances.
\begin{figure}[h]
\centering
\includegraphics[width=0.4\textwidth]{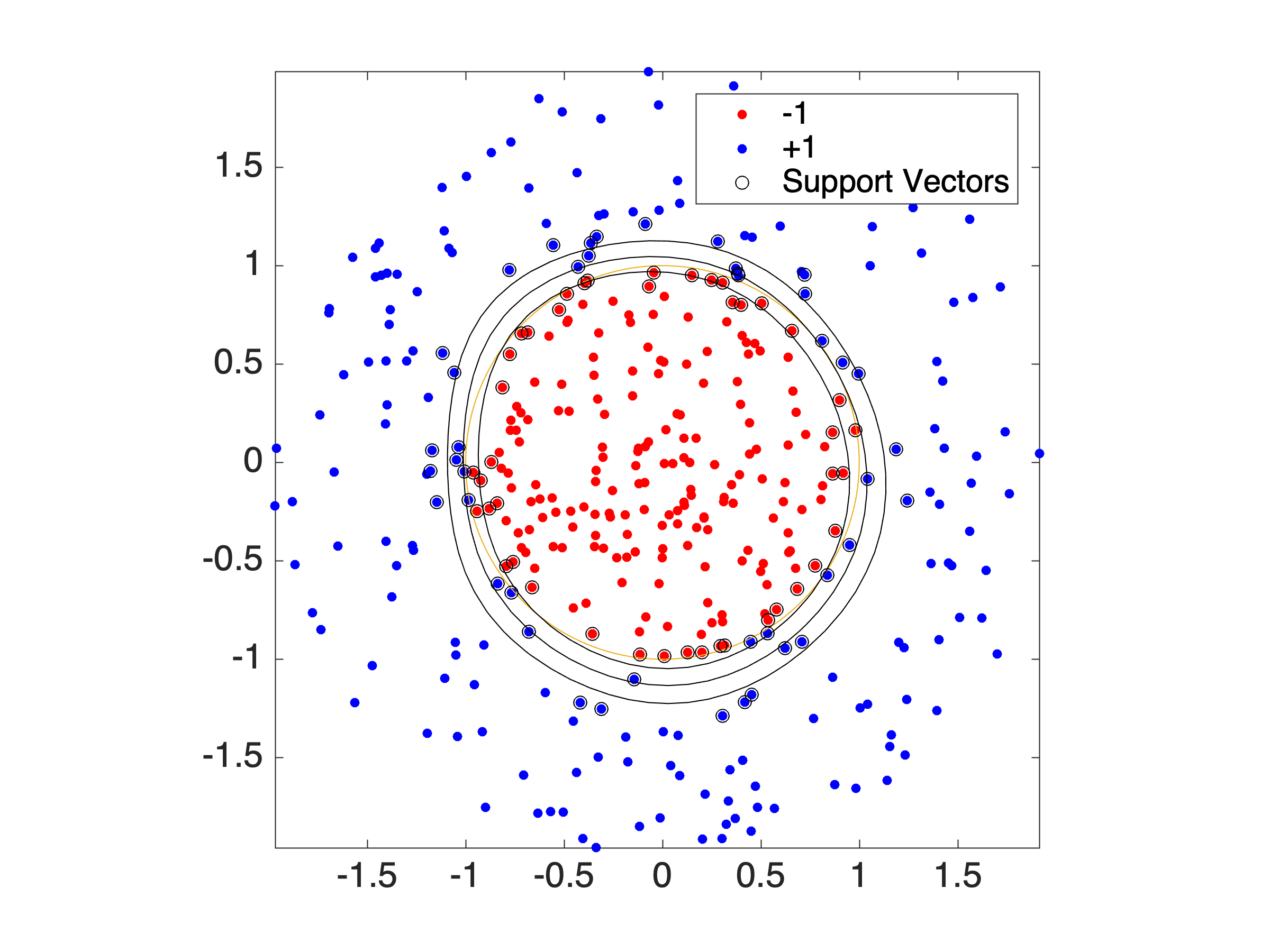}
\caption{prediction contours of standard SVM}
\label{fig:SVM}
\end{figure}

We carried out the same procedures to train the PkSVM model with multiple other datasets. These are illustrated in Figure \ref{fig:eg2} and \ref{fig:eg3}. In particular, the data points in Figure \ref{fig:eg2} are generated in the same way as above except for that the covariances of the red points are
	\[
		\Sigma_R = \left[\begin{matrix}
		0.09 & 0 \\ 0 & 0.01
		\end{matrix}\right],
	\]
and the covariances of the blue points are
	\[
		\Sigma_B = \left[\begin{matrix}
		0.01 & 0 \\ 0 & 0.09
		\end{matrix}\right].
	\]
\begin{figure}[ht]
\centering
\begin{subfigure}{.239\textwidth}
\centering
  \includegraphics[width=1.1\linewidth]{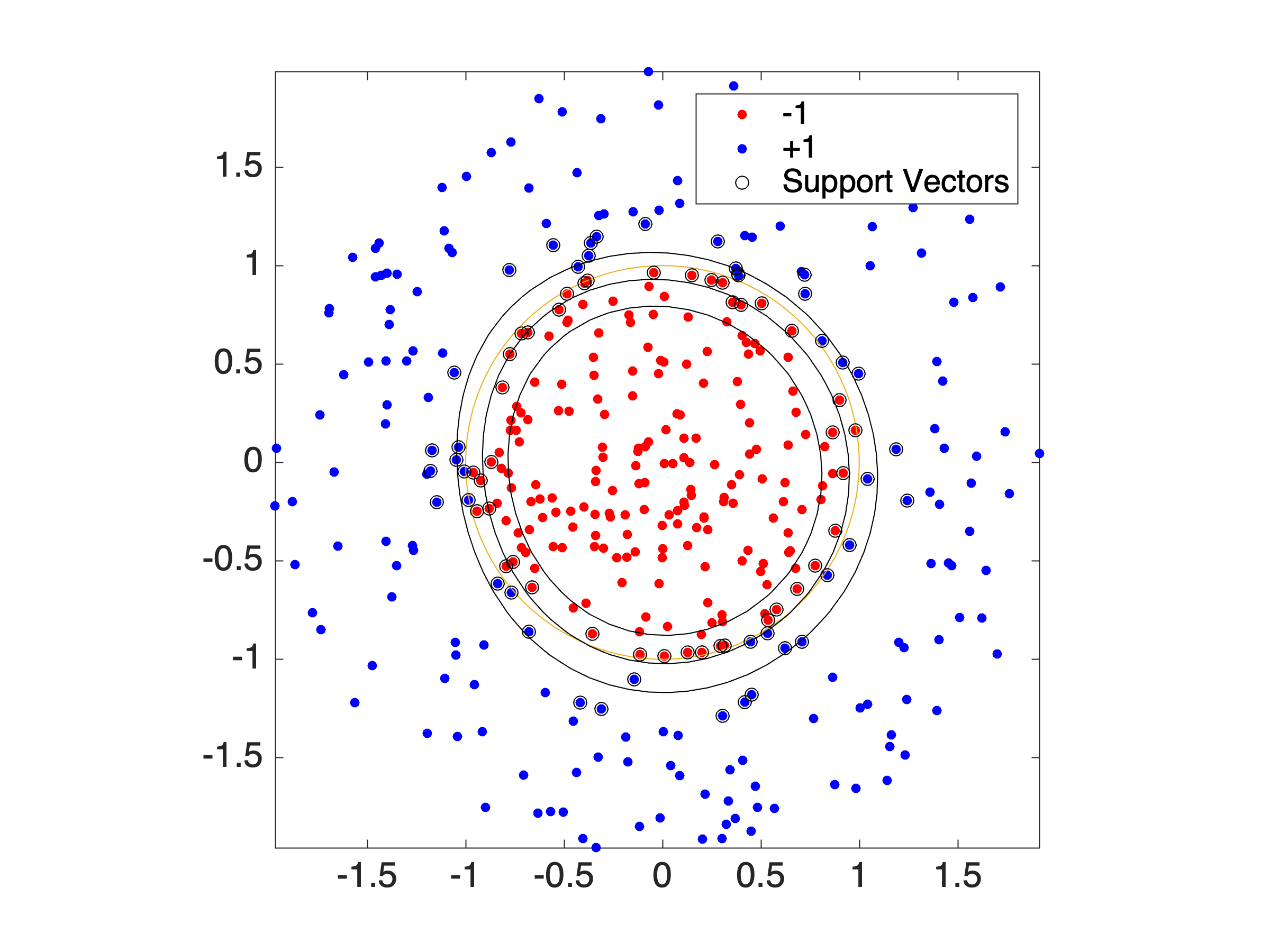}
  \caption{$\Sigma=\Sigma_B$}
  \label{fig:eg2a}
\end{subfigure}
\hspace{-0.5cm}
\begin{subfigure}{.239\textwidth}
\centering
  \includegraphics[width=1.1\linewidth]{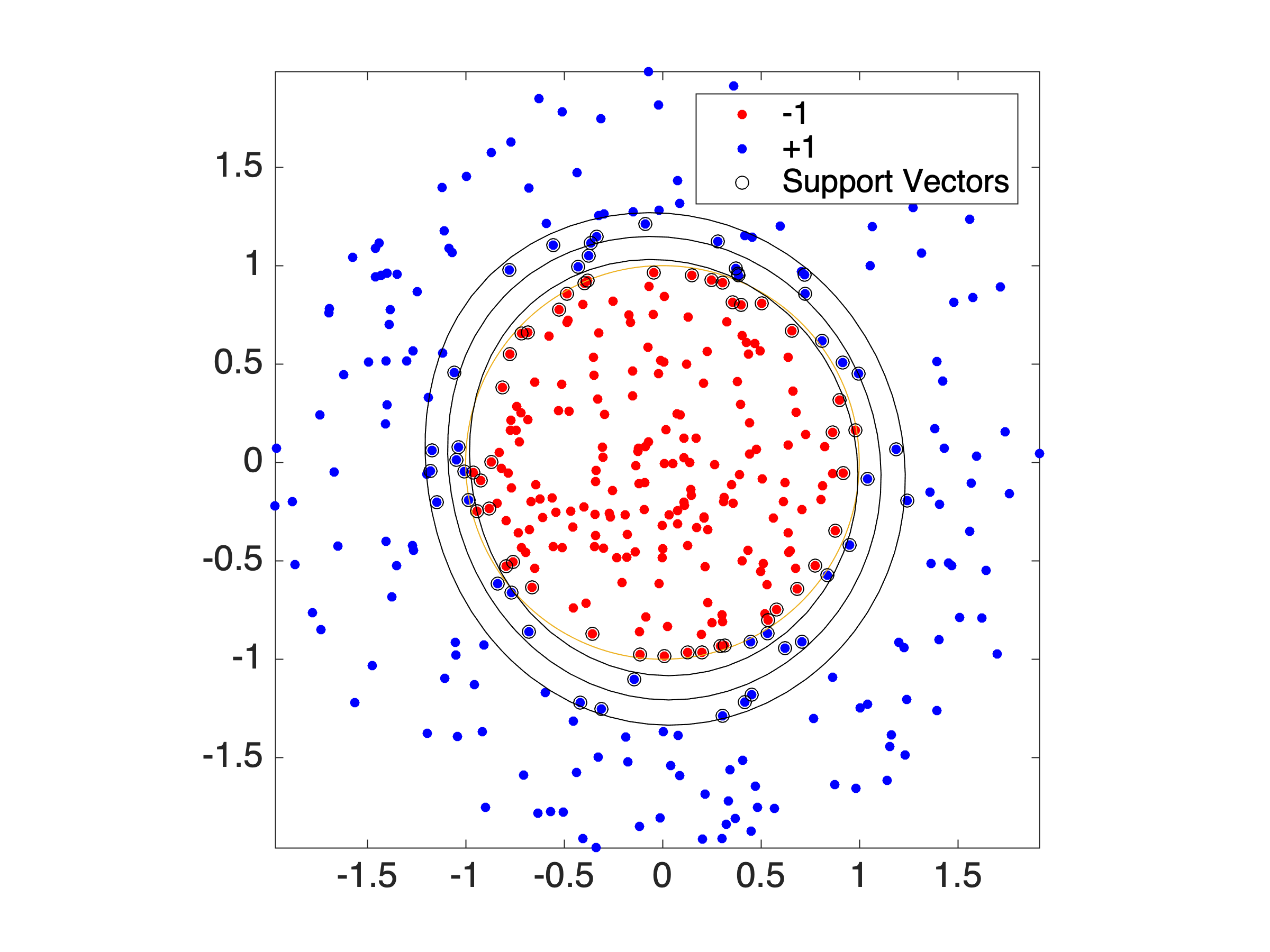}
  \caption{$\Sigma=\Sigma_R$}
  \label{fig:eg2b}
\end{subfigure}
\caption{Prediction contours with different uncertainties}
\label{fig:eg2}
\end{figure}
As can be seen from Figure~\ref{fig:eg2}, when the uncertainties of the testing data are the same as the blue (red) training data, then the prediction boundary tends to go inward (outward), which is consistent with the intuition behind the PkSVM. Similar observations can be drawn in Figure~\ref{fig:eg3} for a different dataset.

\begin{figure}[ht]
\centering
\begin{subfigure}{.239\textwidth}
\centering
  \includegraphics[width=1.1\linewidth]{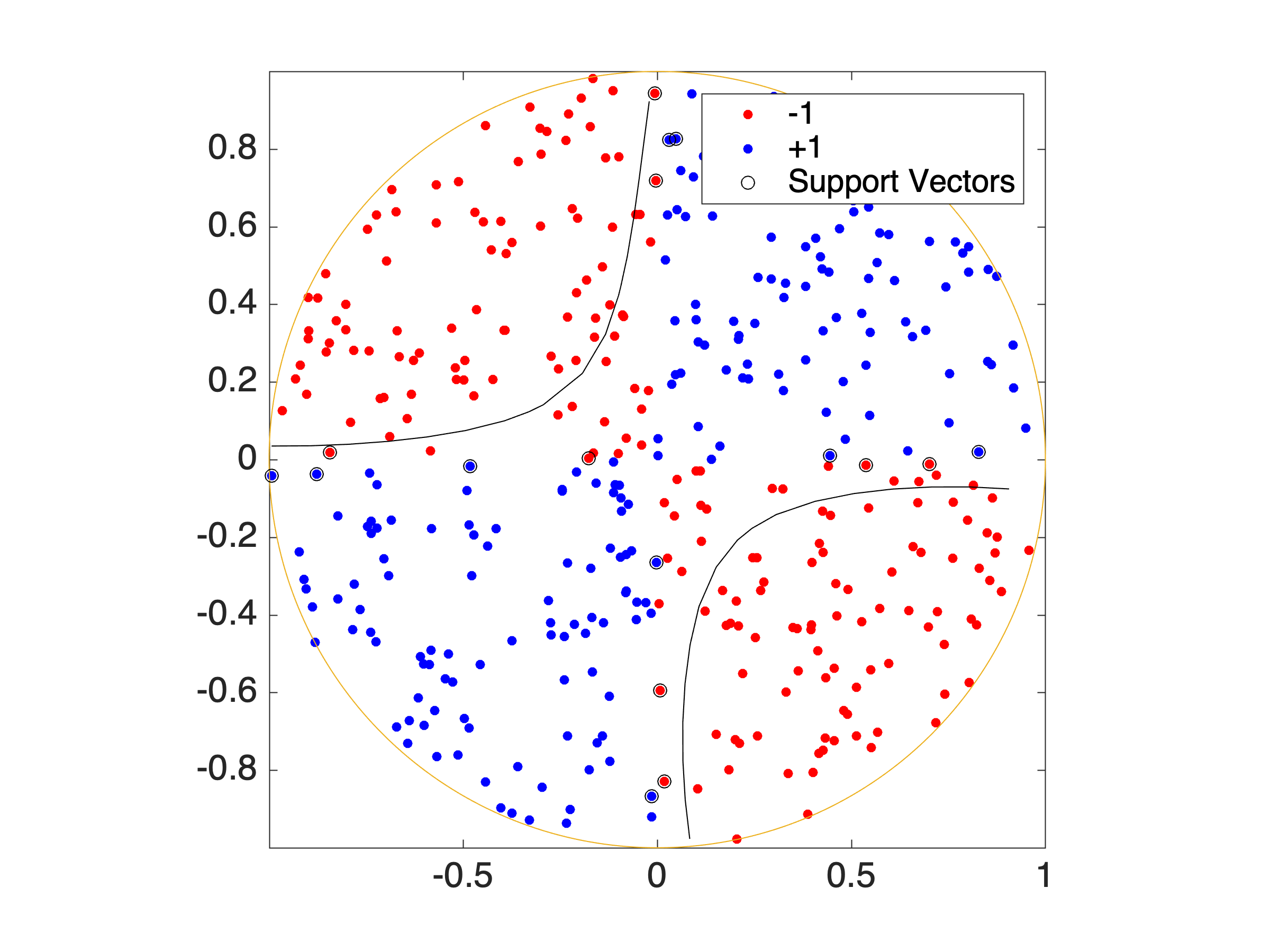}
  \caption{$\Sigma=\Sigma_B$}
  \label{fig:eg3a}
\end{subfigure}
\hspace{-0.5cm}
\begin{subfigure}{.239\textwidth}
\centering
  \includegraphics[width=1.1\linewidth]{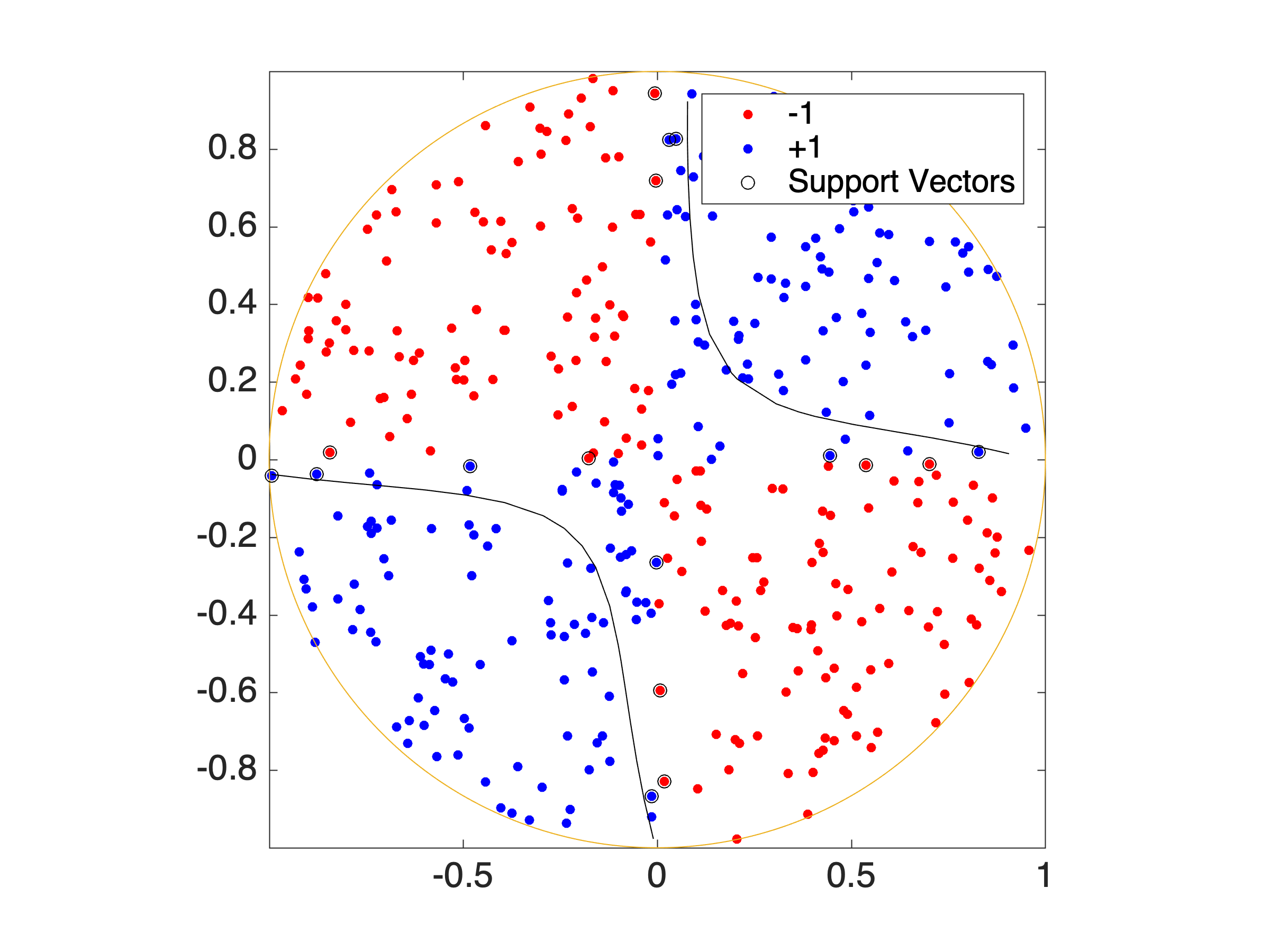}
  \caption{$\Sigma=\Sigma_R$}
  \label{fig:eg3b}
\end{subfigure}
\caption{Prediction contours with different uncertainties}
\label{fig:eg3}
\end{figure}

\section{Conclusion}\label{sec:conclusion}
The formalism herein presents a new paradigm where data incorporate a quantification of their own uncertainty. We focused on binary classification and the case of ``Gaussian points'' to present a proof-of-concept, in the form of a suitable kernel in Theorem~\ref{thm1}. Numerical experiments are provided to illustrate our framework. The basic idea appears to be easily generalizable to more detailed and explicit descriptions of uncertainty, e.g., Gaussian Mixture models, with however, the caveat of added complexity in the resulting formulae. This we expect will be the starting point of future investigations. Interestingly, the Gaussian points data structure coincides with the diffusion tensor widely used in diffusion tensor imaging (DTI) \cite{LeManPouCha01,FarCheGeoLen16}, an important technique in magnetic resonant imaging (MRI). Thus, another promising direction lies in the interface between our framework and DTI.

\section*{Acknowledgments}
This research was funded through NSF under grants 1901599, 1807664, 1839441, 
AFOSR under grant FA9550-20-1-0029, National Institutes of Health grant RF1 AG053991, and the Breast Cancer Research Foundation.

\bibliographystyle{IEEEtran}
\bibliography{refs}
\end{document}